\documentclass[11pt, a4paper]{lumia}

\usepackage[sort&compress]{natbib}
\usepackage{xspace}

\theoremstyle{plain}
\newtheorem{theorem}{Theorem}%[section]
\newtheorem{proposition}[theorem]{Proposition}
\newtheorem*{proposition*}{Proposition}

\theoremstyle{definition}

\theoremstyle{definition}

\def\eqref#1{equation~\ref{#1}}
\usepackage{graphicx}           % graphics support
\usepackage{tikz}               % TikZ graphics
\usepackage[edges]{forest}      % forest trees

\usepackage{url}                % simple URL typesetting
\usepackage{xurl}               % extended URL breaking

\usepackage{array}              % extended array and tabular
\usepackage{longtable}          % multi-page tables
\usepackage{multirow}           % multirow cells
\usepackage{makecell}           % enhanced cell formatting
\usepackage{ragged2e}           % ragged text alignment

\usepackage{mathtools}          % additional math tools
\usepackage{nicefrac}           % compact symbols for 1/2, etc.

\usepackage{algorithm}          % algorithm environment
\usepackage{algorithmicx}       % algorithmic pseudocode
\usepackage{algpseudocode}      % pseudocode formatting
\usepackage{listings}           % code listings

\usepackage{subcaption}         % subfigures and subcaptions
\usepackage{wrapfig}            % text wrapping around figures
\usepackage[export]{adjustbox}  % adjustable boxes

\usepackage{xspace}             % intelligent spacing
\usepackage[normalem]{ulem}     % underlining without affecting emphasis
\usepackage{CJKutf8}            % CJK character support

\usepackage[tikz]{bclogo}       % logo boxes
\usepackage[framemethod=tikz]{mdframed} % framed text

\usepackage{lipsum}             % lorem ipsum text
\usepackage{tocloft}            % table of contents formatting
\usepackage{afterpage}          % page break control
\usepackage{bbding}             % additional symbols
\usepackage{epigraph}           % epigraphs
\usepackage{minitoc}            % mini table of contents
\usepackage{multicol}           % multiple columns
\usepackage{textgreek}          % Greek text
\usepackage[capitalize,noabbrev]{cleveref}
\usepackage{sidecap}

\newcolumntype{P}[1]{>{\RaggedRight\arraybackslash}p{#1}}

\definecolor{uclablue}{RGB}{39, 116, 174}
\definecolor{bigaired}{RGB}{156, 0, 0}
\definecolor{myblue}{HTML}{598BE7}
\definecolor{mildblue}{RGB}{31,119,180}
\definecolor{sectionblue}{RGB}{70, 130, 180}
\definecolor{methodblue}{RGB}{0, 150, 136}
\definecolor{bgblue}{RGB}{245,243,253}
\definecolor{ttblue}{RGB}{91,194,224}
\definecolor{mygreen}{rgb}{0.64, 0.56, 0.88}
\definecolor{myyellow}{rgb}{0.68, 0.6, 0.1}
\definecolor{fancygreen}{rgb}{0.33, 0.68, 0.20}
\definecolor{salmon}{rgb}{0.94, 0.52, 0.49}
\definecolor{tablegreen}{rgb}{0.82, 0.94, 0.75}
\definecolor{tableblue}{rgb}{0.81, 0.90, 0.94}
\definecolor{tablered}{rgb}{0.97, 0.85, 0.85}
\definecolor{tableorange}{rgb}{0.96, 0.85, 0.81}
\definecolor{myorange}{rgb}{1.0, 0.49, 0.0}
\definecolor{tlgreen}{rgb}{0.33, 0.68, 0.20}
\definecolor{darkgreen}{RGB}{0,100,0}
\definecolor{darkred}{RGB}{200, 0, 0}
\definecolor{customyellow}{HTML}{FFFACD}
\definecolor{refinegreen}{RGB}{0, 128, 75}
\definecolor{scoregreen}{RGB}{34, 139, 34}
\definecolor{hidden-blue}{RGB}{194,232,247}
\definecolor{hidden-black}{RGB}{20,68,106}
\definecolor{yes}{HTML}{C6EFCE}
\definecolor{no}{HTML}{FFC7CE}
\definecolor{partial}{HTML}{FFEB9C}
\definecolor{external}{HTML}{D9E1F2}
\definecolor{hdr}{HTML}{F2F2F2}
\definecolor{GRPOrow}{gray}{0.96}
\definecolor{FlowRLrow}{RGB}{225,236,255}
\definecolor{FlowBlue}{RGB}{80,120,210}
\definecolor{GRPOGray}{gray}{0.35}

\hypersetup{
    colorlinks=true, 
    citecolor=uclablue, 
    linkcolor=bigaired,
    urlcolor=darkblue
}

\setlist[itemize]{leftmargin=20pt, noitemsep, topsep=0pt}

\NewDocumentCommand{\kaiyan}{mO{}}{\textcolor{purple}{\textsuperscript{\textit{kaiyan}}\textsf{\textbf{\small[#1]}}}}
\NewDocumentCommand{\yuxin}{mO{}}{\textcolor{cyan}{\textsuperscript{\textit{yuxin}}\textsf{\textbf{\small[#1]}}}}
\NewDocumentCommand{\bx}{mO{}}{\textcolor{green}{\textsuperscript{\textit{bx}}\textsf{\textbf{\small[#1]}}}}
\NewDocumentCommand{\at}{mO{}}{\textcolor{red}{\textsuperscript{\textit{AT}}\textsf{\textbf{\small[#1]}}}}
\NewDocumentCommand{\re}{mO{}}{\textcolor{blue}{\textsuperscript{\textit{RE}}\textsf{\textbf{\small[#1]}}}}
\NewDocumentCommand{\ybsun}{mO{}}{\textcolor{magenta}{\textsuperscript{\textit{youbang}}\textsf{\textbf{\small[#1]}}}}
\NewDocumentCommand{\runze}{mO{}}{\textcolor{orange}{\textsuperscript{\textit{runze}}\textsf{\textbf{\small[#1]}}}}
\NewDocumentCommand{\add}{mO{}}{\textcolor{darkgreen}{\textsuperscript{\textit{Maybe Consider Discuss}}\textsf{\textbf{[#1]}}}}

\newcommand{\cmark}{\textcolor{darkgreen}{\boldmath$\checkmark$}}
\newcommand{\xmark}{\textcolor{darkred}{\boldmath$\times$}}

\newcommand{\ours}{\textsc{CoDAR}~}
\newcommand{\ourslong}{\textbf{Co}ntinuous \textbf{D}iffusion with Contextual \textbf{A}uto\textbf{R}egressive Decoder~}

\newenvironment{itemize*}%
 {\leftmargini=10pt\begin{itemize}%
  \setlength{\itemsep}{0pt}%
  \setlength{\parskip}{0pt}%
  }%
 {\end{itemize}}

\newenvironment{enumerate*}%
 {\begin{enumerate}%
  \setlength{\itemsep}{0pt}%
  \setlength{\parskip}{0pt}}%
 {\end{enumerate}}

\newcommand{\cellstatus}[1]{%
  \begingroup
  \StrTrim{#1}[\statusval]%
  \IfStrEq{\statusval}{Yes}{\cellcolor{yes}\cmark}{}%
  \IfStrEq{\statusval}{No}{\cellcolor{no}\xmark}{}%
  \IfBeginWith{\statusval}{Yes (}{\cellcolor{yes}\cmark~\textit{\statusval\unskip}}{}%
  \IfStrEq{\statusval}{Partial}{\cellcolor{partial}\textbf{Partial}}{}%
  \IfStrEq{\statusval}{External}{\cellcolor{external}\textbf{External}}{}%
  \endgroup
}

\newtcolorbox{myboxi}[1][]{
  breakable,
  title=#1,
  colback=red!5,
  colbacktitle=red!5,
  coltitle=black,
  fonttitle=\bfseries,
  bottomrule=0pt,
  toprule=0pt,
  leftrule=2pt,
  rightrule=2pt,
  titlerule=0pt,
  arc=0pt,
  outer arc=0pt,
  colframe=red,
}

\newtcolorbox{myboxnote}[1][]{
  breakable,
  title=#1,
  colback=orange!0,
  colbacktitle=orange!0,
  coltitle=black,
  fonttitle=\bfseries,
  bottomrule=0pt,
  toprule=0pt,
  leftrule=2pt,
  rightrule=2pt,
  titlerule=0pt,
  arc=0pt,
  outer arc=0pt,
  colframe=orange,
}

\newtcolorbox{myboxii}[1][]{
  breakable,
  freelance,
  title=#1,
  colback=white,
  colbacktitle=white,
  coltitle=black,
  fonttitle=\bfseries,
  bottomrule=0pt,
  boxrule=0pt,
  colframe=white,
  overlay unbroken and first={
  \draw[red!75!black,line width=3pt]
    ([xshift=5pt]frame.north west) -- 
    (frame.north west) -- 
    (frame.south west);
  \draw[red!75!black,line width=3pt]
    ([xshift=-5pt]frame.north east) -- 
    (frame.north east) -- 
    (frame.south east);
  },
  overlay unbroken app={
  \draw[red!75!black,line width=3pt,line cap=rect]
    (frame.south west) -- 
    ([xshift=5pt]frame.south west);
  \draw[red!75!black,line width=3pt,line cap=rect]
    (frame.south east) -- 
    ([xshift=-5pt]frame.south east);
  },
  overlay middle and last={
  \draw[red!75!black,line width=3pt]
    (frame.north west) -- 
    (frame.south west);
  \draw[red!75!black,line width=3pt]
    (frame.north east) -- 
    (frame.south east);
  },
  overlay last app={
  \draw[red!75!black,line width=3pt,line cap=rect]
    (frame.south west) --
    ([xshift=5pt]frame.south west);
  \draw[red!75!black,line width=3pt,line cap=rect]
    (frame.south east) --
    ([xshift=-5pt]frame.south east);
  },
}

\tcbset{
  takeawaysbox/.style={
    title=Takeaways,
    colback=lightblue!80,
    colframe=black,
    fonttitle=\bfseries\small,
    coltitle=white,
    colbacktitle=black,
    enhanced,
    attach boxed title to top left={xshift=2.5mm,yshift=-2.5mm},
    boxed title style={rounded corners, size=small, colframe=black, colback=black},
    width=\linewidth,
    arc=3.5mm
  }
}

\mdfdefinestyle{mystyle}{%
  rightline=true,
  innerleftmargin=10,
  innerrightmargin=10,
  outerlinewidth=3pt,
  topline=false,
  rightline=true,
  bottomline=false,
  skipabove=\topsep,
  skipbelow=\topsep
}

\tikzset{%
    every node/.style={font=\tiny},
    parent/.style =          {align=center,text width=2cm,rounded corners=3pt, line width=0.3mm, fill=gray!10,draw=gray!80},
    child/.style =           {align=center,text width=2.0cm,rounded corners=3pt, fill=blue!10,draw=blue!80,line width=0.3mm},
    grandchild/.style =      {align=center,text width=2cm,rounded corners=3pt},
    greatgrandchild/.style = {align=center,text width=1.5cm,rounded corners=3pt},
    greatgrandchild2/.style = {align=center,text width=1.5cm,rounded corners=3pt},    
    referenceblock/.style =  {align=center,text width=1.5cm,rounded corners=2pt},
    pretrain/.style =           {align=center,text width=2.0cm,rounded corners=3pt, fill=blue!10,draw=blue!80,line width=0.3mm},   
    pretrain_work/.style =           {align=center, text width=8.5cm,rounded corners=3pt, fill=blue!10,draw=blue!0,line width=0.3mm},  
    template/.style =           {align=center,text width=2.0cm,rounded corners=3pt, fill=red!10,draw=red!80,line width=0.3mm},   
    template_work/.style =           {align=center,text width=8.5cm,rounded corners=3pt, fill=red!10,draw=red!0,line width=0.3mm},    
    answer/.style =           {align=center,text width=2.0cm,rounded corners=3pt, fill= cyan!10,draw= cyan!80,line width=0.3mm},   
    answer_work/.style =           {align=center,text width=8.5cm,rounded corners=3pt, fill= cyan!10,draw= cyan!0,line width=0.3mm},      
    multiple/.style =           {align=center,text width=2.0cm,rounded corners=3pt, fill= orange!10,draw= orange!80,line width=0.3mm},   
    multiple_work/.style =           {align=center,text width=8.5cm,rounded corners=3pt, fill= orange!10,draw= orange!0,line width=0.3mm},        
    tuning/.style =           {align=center,text width=2.0cm,rounded corners=3pt, fill= magenta!10,draw= magenta!80,line width=0.3mm},   
    tuning_work/.style =           {align=center,text width=8.5cm,rounded corners=3pt, fill= magenta!10,draw= magenta!0,line width=0.3mm},          
}

\newcommand{\lstbg}[3][0pt]{{\fboxsep#1\colorbox{#2}{\strut #3}}}

\lstdefinelanguage{diff}{
  basicstyle=\ttfamily\small,
  morecomment=[f][\lstbg{red!20}]-,
  morecomment=[f][\lstbg{green!20}]+,
}

\lstdefinelanguage{diffpython}{
  language=diff,
  morekeywords={def, if, else, for, while, return, import, from, as, class, with, try, except, finally, raise, lambda, and, or, not, in, is, None, True, False},
  morecomment=[l]{\#},
  morestring=[b]",
  morestring=[b]',
}

\setheadertext{LUMIA Lab}

\correspondingemail{\emailicon$^{1}$ junzheshen@sjtu.edu.cn \\ \emailicon$^{\ddagger}$ lin.zhouhan@gmail.com \quad $^\ddagger$ Corresponding Author. }

\title{CoDAR: Continuous Diffusion Language Models are More Powerful Than You Think}
\setheadertitle{Continuous Diffusion Language Models are More Powerful Than You Think}

\author{%
  \Authfont Junzhe Shen$^{1}$, Jieru Zhao$^{3}$, Ziwei He$^{2}$, Zhouhan Lin$^{1234\ddagger}$\\
  $^1$ LUMIA Lab, School of Artificial Intelligence, Shanghai Jiao Tong University \\
  $^2$ Shanghai Innovation Institute \\
  $^3$ School of Computer Science, Shanghai Jiao Tong University \\
  $^4$ Shanghai AI Laboratory
}

\begin{document}

% ====================
% ABSTRACT
% ====================
\begin{abstract}
We study why continuous diffusion language models (DLMs) have lagged behind discrete diffusion approaches despite their appealing continuous generative dynamics.
  Under a controlled token-recovery study, we identify token rounding, the final projection from denoised embeddings to tokens, as a primary bottleneck.
  Building on these insights, we propose \ours~(\ourslong), a two-stage framework that keeps diffusion entirely continuous in an embedding space while learning a strong, context-conditional discretizer: an autoregressive Transformer decoder that cross-attends to the denoised embedding sequence and performs contextualized rounding to tokens. Experiments on LM1B and OpenWebText demonstrate that \ours substantially improves generation quality over latent diffusion and becomes competitive with strong discrete DLMs, while exposing a simple decoder-temperature knob to navigate the fluency–diversity trade off.
\end{abstract}

% Generate title with LUMIA style formatting
\maketitle

% ====================
% YOUR PAPER CONTENT GOES HERE
% ====================
\begin{figure}[h]
\centering
\begin{minipage}{0.5\columnwidth}
    \includegraphics[width=\columnwidth]{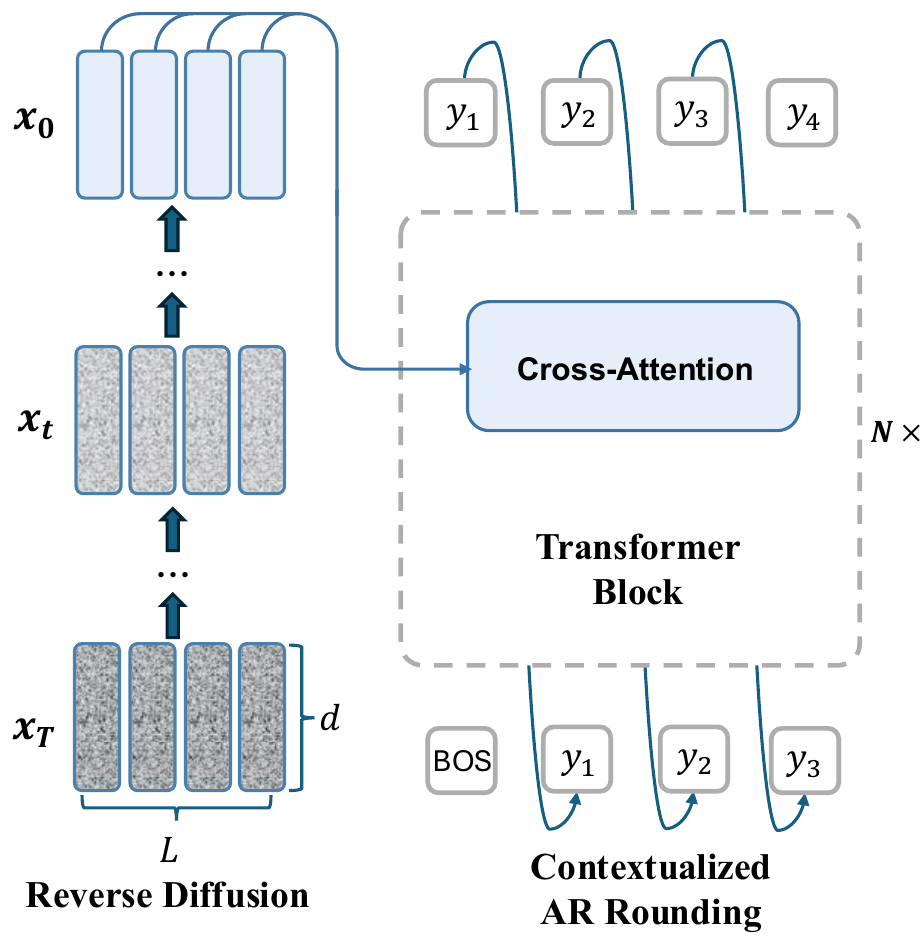}
\end{minipage}
% \hfill
\hspace{15pt}
\begin{minipage}{0.3\columnwidth}
    \caption{Framework of \ours. Starting from a noisy latent sequence $x_T$, a reverse diffusion process progressively denoises the hidden states to $x_0$. $\mathbf{x}_T,\ldots, \mathbf{x}_0 \in \mathrm{R}^{L\times d}$, where $L$ denotes the sequence length and $d$ denotes the size of hidden states. After that, an autoregressive Transformer decoder conditions on the denoised $\mathbf{x}_0$ with cross-attention to translate $\mathbf{x}_0$ to discrete tokens $\mathbf{y}_1, \ldots, \mathbf{y}_L$.}
    \label{fig:architecture}
\end{minipage}
\end{figure}
% \begin{figure}[h]
%     \centering
%     \includegraphics[width=0.55\columnwidth]{Figures/figrue1.pdf}
%     \caption{Framework of \ours. Starting from a noisy latent sequence $x_T$, a reverse diffusion process progressively denoises the hidden states to $x_0$. $\mathbf{x}_T,\ldots, \mathbf{x}_0 \in \mathrm{R}^{L\times d}$, where $L$ denotes the sequence length and $d$ denotes the size of hidden states. After that, an autoregressive Transformer decoder conditions on the denoised $\mathbf{x}_0$ with cross-attention to translate $\mathbf{x}_0$ to discrete tokens $\mathbf{y}_1, \ldots, \mathbf{y}_L$ }
%     \label{fig:architecture}
% \end{figure}
\section{Introduction}
Continuous diffusion models have achieved remarkable success in domains such as image generation and structured latent spaces, where they demonstrate strong modeling capacity and sample quality. However, their adoption and scalability in natural language processing remain limited. In contrast to continuous media, language is inherently discrete, leading to fundamental challenges when applying continuous generative processes directly to text. While continuous formulations offer theoretical advantages~\citep{zhou2025ccdd, pynadath2025candi}, including stronger theoretical expressivity and smoother latent reasoning, continuous space diffusion language models (DLMs) have fallen behind discrete DLMs~\citep{diffusionbert, survey}.

A core difficulty lies in the misalignment between continuous diffusion and the discrete nature of language: diffusion processes typically operate over continuous spaces, whereas linguistic units such as tokens are inherently categorical. Prior works on continuous DLMs have attempted to bridge this gap using a rounding step~\citep{diffusionlm, sed, genie, difformer} or formulating diffusion in the logit space~\citep{tess, tess2}. In parallel, embedding-based latent diffusion methods such as \textit{Latent Diffusion for Language Generation}~\citep{ld4lg} learn diffusion in the latent space of a language autoencoder built on a pretrained encoder--decoder LM. While this provides a principled continuous space that is decodable by construction, it ties the method to encoder--decoder language models and thus narrows the choice of representations.

Beyond the continuous--discrete mismatch, a second difficulty lies in the high dimensionality of latent representations. In latent diffusion for images, recent studies document an optimization dilemma where increasing latent capacity (often via higher-dimensional latents) improves reconstruction but can hinder diffusion training and generation quality unless additional alignment or regularization is introduced~\citep{yao2025reconstruction, lai2025toward, rae}. Although language differs from vision, analogous pressures can arise when diffusing over high-dimensional embedding sequences. Due to these difficulties, research largely pivoted toward discrete modeling strategies that match the token-level structure of language~\citep{sedd, mdlm} as the field matured.

We argue that the performance gap between continuous DLMs and discrete DLMs is not solely due to the diffusion objective itself, but rather to a \emph{rounding bottleneck}: the challenge of mapping noisy continuous embeddings back to discrete tokens in a large space. Most existing embedding-space DLMs typically rely on a rounding operator to recover tokens~\citep{diffusionlm, genie}. Such operators treat each position largely independently and provide limited ability to leverage linguistic context when the denoised embedding is ambiguous or off manifold. Conversely, discrete or simplex-space DLMs avoid explicit rounding but must operate directly in a categorical state space, which changes the learning dynamics and often shifts complexity into the diffusion transition design and sampling procedure~\citep{sedd, mdlm, tess2}.%These observations suggest that \emph{continuous diffusion and discrete decoding \redmark{should be} disentangled}: diffusion can focus on modeling a continuous latent sequence, while a dedicated, context-aware decoder performs the discrete projection. 

In this work, we analyze the rounding bottleneck of continuous DLMs both theoretically and empirically, and propose \ours(\ourslong), a novel framework that effectively resolves the two aforementioned problems by \emph{keeping diffusion entirely continuous} while \emph{learning a powerful, context-conditional rounding module}. Concretely, \ours factorizes generation into (i) a continuous diffusion process in an embedding space that is favorable for diffusion, and (ii) an autoregressive Transformer decoder that maps the generated continuous states to discrete tokens via cross-attention. This design keeps the diffusion component simple and fully continuous, and delegates the hardest part of decoding to a model class that is known to excel at sequence transduction.

We summarize our main contributions as follows:
\begin{itemize}
    % \item We identify \textbf{token decoding}\redmark{especially under low dim hidden states} as a principal bottleneck for continuous embedding DLMs and provide empirical evidence that naive pointwise classifiers are insufficient for accurate decoding, even when given strong hidden representations (\cref{sec:empirical}).
    \item We identify token rounding, especially under low dimensional hidden states, as a principal bottleneck for continuous embedding DLMs theoretically and empirically, showing why pointwise classifiers can be suboptimal and validating this with controlled token-recovery experiments.
    \item We propose \ours, a two-stage continuous diffusion language modeling framework: a continuous diffusion generator produces embedding sequences, and an autoregressive Transformer decoder performs contextual rounding back to tokens.
    \item We show that \ours improves generation quality over latent diffusion, and closes the gap between discrete DLMs.
\end{itemize}

\section{Related Work}
\subsection{Continuous Diffusion Language Models}

Diffusion-LM~\citep{diffusionlm} models sequences as Gaussian vectors that are iteratively denoised into word vectors, enabling plug-and-play controllable generation. Building on the embedding-space perspective, Self-conditioned Embedding Diffusion~\citep{sed} introduces self-conditioning for diffusion over token embeddings and demonstrates competitive generation quality with potential inference efficiency benefits. Difformer~\citep{difformer} analyzes optimization pathologies in embedding diffusion (e.g., embedding collapse and schedule issues) and proposes anchor loss and noise rescaling to stabilize training and improve generation across tasks. GENIE~\citep{genie} frames diffusion for pre-training, using an encoder plus diffusion-based decoder and a continuous paragraph denoise objective to reconstruct coherent paragraphs from corrupted inputs. LD4LG~\citep{ld4lg} trains diffusion in the latent space of an encoder–decoder language autoencoder, sampling compact continuous latents that are decoded by a pretrained decoder. More recent variants explore where to diffuse and how to better respect token discreteness: TESS~\citep{tess} performs diffusion on the logit simplex (rather than learned embeddings) with self-conditioning for fully non-autoregressive text-to-text generation. Smoothie~\citep{shabalin2025smoothie} proposes progressively smoothing token embeddings by semantic similarity to combine semantic structure with a more natural decoding process.

% \subsection{Discrete Diffusion Language Models}
% Discrete diffusion language models extend diffusion to categorical token spaces. D3PM~\citep{austin2021structured} formalize this view with flexible categorical transition matrices. SEDD\citep{sedd} proposes score entropy to estimate data-distribution ratios in discrete spaces, reporting large improvements over earlier diffusion paradigms. MDLM~\citep{mdlm} demonstrates that, with a modern training recipe and a Rao–Blackwellized objective, masked diffusion can substantially close the gap to autoregressive LMs. GIDD\citep{rutte2025generalized} generalizes masked diffusion by allowing the noising process to be expressed as a linear interpolation between data and a (possibly time-varying) mixing distribution, enabling hybrid corruption schemes that improve sample quality and restore the model’s ability to revise earlier mistakes. These algorithmic advances are now being pushed to scale, with LLaDA~\citep{nie2025llada} demonstrating diffusion LMs trained from scratch under a pre-training + SFT pipeline, Seed Diffusion~\citep{song2025seed} emphasizing large-scale diffusion decoding with high-speed inference, LLaDA2.0~\citep{bie2025llada20} reporting diffusion language models scaled up to 100B parameters, and Google DeepMind’s Gemini Diffusion~\citep{deepmind2025geminidiffusion} exploring diffusion-based text generation aimed at greater speed and user control.

\subsection{Hybrid Architectures}
\paragraph{AR-Diffusion Hybrid}A growing line of work explores hybrid diffusion–autoregressive architectures that aim to combine diffusion’s global refinement/parallelism with AR decoding’s fluency and KV-cache-friendly generation. AR-Diffusion~\citep{wu2023ardiffusion} injects causal, left-to-right structure into diffusion by using a position-dependent (dynamic) number of denoising steps so left tokens “settle” earlier and condition later ones. DGLM~\citep{lovelace2024dglm} uses a diffusion model to generate continuous semantic proposals (soft prompts) that steer a strong AR LM toward desired attributes. SDLM~\citep{liu2025sequential} further “retrofit” pretrained AR LMs by performing diffusion within masked blocks but decoding consecutive subsequences adaptively (via Next Sequence Prediction) to maintain KV-cache compatibility and handle variable uncertainty across the sequence. Moving toward single-model synergy, TiDAR~\citep{liu2025tidar} explicitly separates roles by drafting in diffusion while sampling final outputs autoregressively using structured attention masks to achieve high throughput without sacrificing AR-quality. For reasoning-centric settings, LaDiR~\citep{kang2025ladir} augments an existing LLM with a VAE-defined latent “thought” space and a latent diffusion model that iteratively refines blockwise reasoning trajectories, enabling more holistic revision than pure AR chain-of-thought. Planner and Executor~\citep{berrayana2025plannerexecutor} studies explicit collaboration where a discrete diffusion model plans and an AR model executes, showing that shifting diffusion to AR communication from text to latent space via a learned projector can markedly improve reasoning accuracy while reducing token-cost.
\paragraph{Continuous-Discrete Hybrid} Hybrid continuous–discrete diffusion LMs model tokens and continuous representations jointly, rather than diffusing in only one space. CCDD~\citep{zhou2025ccdd} co-evolves discrete tokens and continuous states in a single joint diffusion, aiming to preserve continuous latent expressivity while improving trainability via explicit discrete structure. CANDI~\citep{pynadath2025candi} decouples discrete identity corruption from continuous geometric degradation to avoid a mismatch in effective noise regimes, enabling useful continuous gradients while retaining conditional structure. CADD~\citep{cadd} augments mask-based discrete diffusion with a paired continuous latent that replaces the [MASK] “information void” with noisy but informative vectors that guide denoising and offer controllable diversity–precision trade-offs. Compared with these interleaved hybrid processes, our approach keeps diffusion entirely continuous in embedding space and delegates discretization to a separate contextualized decoder.

\section{Theoretical Analysis}%像实验的题目
\label{sec:empirical}
% \subsection{Theoretical Analysis}
Let $Y=(Y_1,\dots,Y_L)$ be a length-$L$ token sequence and let $X\in\mathbb{R}^{L\times d}$ denote the denoised continuous sequence produced by the diffusion generator. Any rounding (discretization) procedure is implicitly performing posterior inference:
\begin{equation}
\hat{y}\in \arg\max_{y}\; p(y\mid X).
\end{equation}
Many embedding-space DLMs implement rounding with a position-wise linear head~\citep{diffusionlm,difformer,genie}, which corresponds to the approximate factorization $p(y\mid X)\approx \prod_{i=1}^L p(y_i\mid X_i)$, treating token recovery as independent classification at each position $i$.

\paragraph{Entropy and conditional total correlation.}
Throughout, $H(\cdot)$ denotes Shannon entropy (in nats when $\log$ is natural). For discrete random variables,
\begin{equation}
H(Y\mid X)\;=\;\mathbb{E}_{x}\Big[-\sum_{y} p(y\mid x)\log p(y\mid x)\Big],
\end{equation}
and similarly $H(Y_i\mid X_i)$ is the remaining uncertainty of $Y_i$ after observing only the local vector $X_i$.
The \emph{conditional total correlation} (conditional TC) measures residual dependence among $(Y_1,\dots,Y_L)$ given $X$:
\begin{equation}
\begin{aligned}
\mathrm{TC}(Y\mid X)
&=
\mathbb{E}_{x}\Big[
D_{\mathrm{KL}}\Big(p(y\mid x)\;\Big\|\;\prod_{i=1}^L p(y_i\mid x)\Big)
\Big] \\
&=
\sum_{i=1}^L H(Y_i\mid X)\;-\;H(Y\mid X)
\;\ge\;0.
\end{aligned}
\end{equation}

\paragraph{Locality gap vs.\ dependence gap.}
Because $X$ contains (weakly) more information than $X_i$, conditioning reduces entropy:
$H(Y_i\mid X)\le H(Y_i\mid X_i)$. This yields a useful decomposition:

\begin{equation}
\begin{aligned}
& \sum_{i=1}^L H(Y_i\mid X_i)\;-\;H(Y\mid X) =
\underbrace{\Big(\sum_{i=1}^L H(Y_i\mid X)\;-\;H(Y\mid X)\Big)}_{\mathrm{TC}(Y\mid X)} + \underbrace{\sum_{i=1}^L\Big(H(Y_i\mid X_i)-H(Y_i\mid X)\Big)}_{\text{locality gap}} 
\ge\;\mathrm{TC}(Y\mid X).
\label{eq:tc_locality_decomp}
\end{aligned}
\end{equation}
Intuitively, $\mathrm{TC}(Y\mid X)$ captures \emph{intrinsic sequence coupling} (syntax/semantics, long-range constraints) that remains even if the decoder sees the entire $X$, whereas the locality gap captures extra uncertainty introduced when a decoder is restricted to per-position evidence $X_i$ rather than the full context $X$.
\begin{proposition}[Optimality gap of pointwise decoding]
\label{prop}
Let $\mathcal{D}_{\mathrm{pw}}$ be the set of all decoders that factorize as $\prod_{i=1}^L q_i(y_i\mid X_i)$, and let $\mathcal{D}_{\mathrm{seq}}$ be the set of all conditional sequence decoders $q(y\mid X)$. Consider the expected negative log-likelihood (NLL) risk
\begin{equation}
\mathcal{R}(q)=\mathbb{E}_{(X,Y)}[-\log q(Y\mid X)].
\end{equation}
Then,
\begin{equation}
\begin{aligned}
\min_{q\in \mathcal{D}_{\mathrm{pw}}}\mathcal{R}(q)\;-\;\min_{q\in \mathcal{D}_{\mathrm{seq}}}\mathcal{R}(q)
&\;=\;
\sum_{i=1}^L H(Y_i\mid X_i)\;-\;H(Y\mid X) \\
&\;\ge\;
\mathrm{TC}(Y\mid X)
\;\ge\;0.
\label{eq:opt_gap}
\end{aligned}
\end{equation}
\end{proposition}
\paragraph{Proof sketch.}
The Bayes-optimal conditional model for $\mathcal{D}_{\mathrm{seq}}$ is $q^*(y\mid X)=p(y\mid X)$, achieving risk $H(Y\mid X)$. For $\mathcal{D}_{\mathrm{pw}}$, the minimizer is $q_i^*(y_i\mid X_i)=p(y_i\mid X_i)$ independently for each $i$, achieving $\sum_i H(Y_i\mid X_i)$. Subtracting yields~\eqref{eq:opt_gap}. The lower bound by $\mathrm{TC}(Y\mid X)$ follows from~\eqref{eq:tc_locality_decomp}. Refer to Appendix \ref{sec:proof} for detailed proof.
% The perspective that factorized training drops dependencies measured by conditional TC is closely related to NAT learning theory: :contentReference[oaicite:1]{index=1}

\paragraph{Interpretation and implications for rounding.}
Equation~\eqref{eq:opt_gap} formalizes two reasons why a linear LM head can be brittle for rounding diffusion outputs:
(i) \emph{sequence dependence:} the gap vanishes only if $Y_1,\dots,Y_L$ are conditionally independent given $X$ (i.e., $\mathrm{TC}(Y\mid X)=0$), which is rarely true for natural language;
(ii) \emph{local evidence restriction:} even if a full-context decoder could in principle exploit $X$ to resolve ambiguity, a per-position head that only sees $X_i$ incurs the locality gap in~\eqref{eq:tc_locality_decomp}.
In realistic diffusion sampling, $X$ is an imperfect denoised embedding sequence: it may be off-manifold, slightly inconsistent across positions, or contain structured errors. Such imperfections typically increase both conditional dependence (larger $\mathrm{TC}(Y\mid X)$) and locality gap, making point-wise decoding strictly suboptimal.

\paragraph{Why increasing $d$ helps but does not eliminate the issue.}
A larger embedding dimension $d$ increases the capacity of each $X_i$ to encode information about $Y_i$. Formally,
\begin{equation}
H(Y_i\mid X_i)\;=\;H(Y_i)\;-\;I(Y_i;X_i),
\end{equation}
so improving the mutual information $I(Y_i;X_i)$ (often easier when $d$ is larger) reduces \emph{marginal} ambiguity.  However, lowering these marginal ambiguities does not imply that sequence recovery becomes pointwise: the approximation $p(y\mid X)\approx \prod_i p(y_i\mid X_i)$ would require that (i) tokens become conditionally independent given the full $X$, and (ii) local evidence $X_i$ is as informative as the global context $X$. In general, neither holds. Moreover, even if the clean embedding manifold were perfectly invertible, diffusion-generated $X$ can be slightly off-manifold. In that case, $X$ may be consistent with multiple nearby token sequences. Choosing the one that is globally coherent requires reasoning over the whole sequence, not each position independently. This explains why improving per-position separability (e.g., by increasing $d$) can improve linear head accuracy, yet still leaves a substantial gap to contextual decoding.
% resolving which globally consistent sentence best explains such $X$ is inherently a \emph{sequence-level} inference problem. This explains why improving per-position separability (e.g., by increasing $d$) can improve linear-head accuracy, yet still leaves a substantial gap to contextual decoding.

\begin{figure}[h!]
    \centering
    \includegraphics[width=0.5\columnwidth]{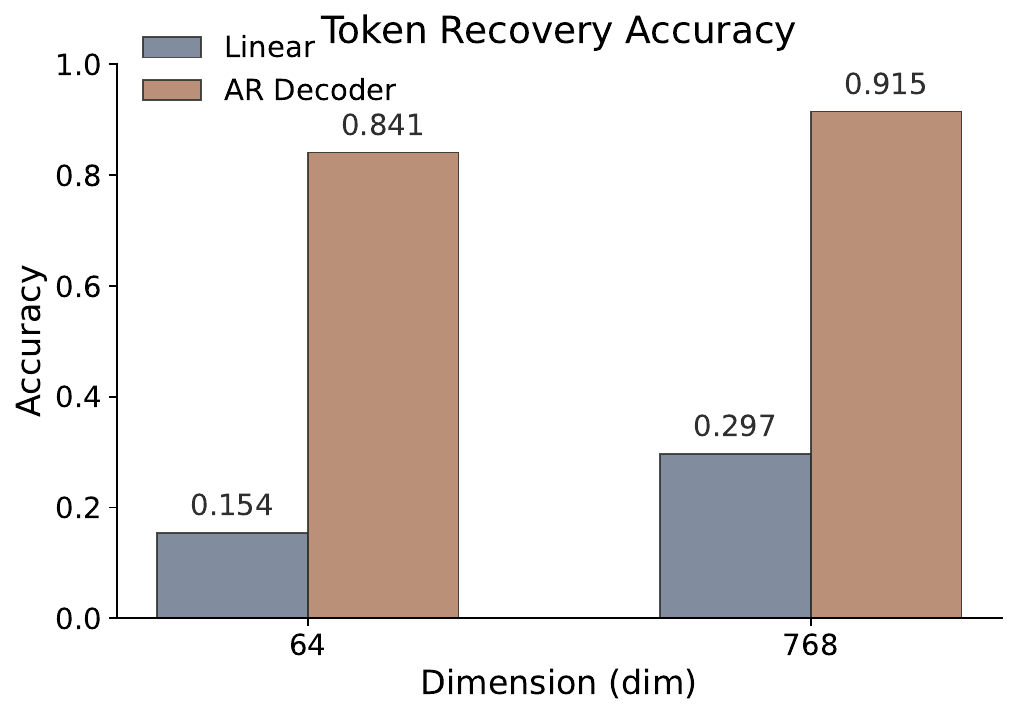}
    \caption{Token recovery rate of point-wise linear classifier and autoregressive Transformer decoder under different sizes of hidden states.}
    \label{fig:token_recovery}
\end{figure}

To test whether the theoretical gap in~\eqref{eq:opt_gap} is practically significant, we run a controlled study. We train two families of token decoders to map hidden states $h_i$ back to tokens $x_i$:
(i) a position-wise linear classifier $p(x_i \mid h_i)$; and
(ii) an autoregressive Transformer decoder that predicts $x_i$ conditioned on previously recovered tokens and cross-attends to the full continuous sequence $H = [h_1,\ldots,h_L]$. 
For each position, we select the token with highest predicted probability as the predicted token, and token recovery is computed as the rate at which the predicted token matches the true token.
The token recovery accuracy in shown in \cref{fig:token_recovery}.

\paragraph{Key observation.}
As shown in \cref{fig:token_recovery}, the Transformer decoder recovers tokens with high accuracy across both low- and high-dimensional representations (e.g., $0.841$ at $d{=}64$ and $0.915$ at $d{=}768$), while the Linear baseline performs poorly ($0.154$ and $0.297$, respectively). This is consistent with~\eqref{eq:opt_gap}: a linear head is constrained to $\mathcal{D}_{\mathrm{pw}}$ and therefore cannot exploit sequence-level coupling (nonzero $\mathrm{TC}(Y\mid X)$) nor full-context evidence (locality gap), both of which are crucial for reliable rounding when $h_i$ are imperfect.

These results motivate our two-stage approach, where continuous diffusion generation and contextual token rounding are decoupled and separately learned:
\begin{itemize}
    \item a continuous diffusion model generates a sequence of token embeddings in $\mathbb{R}^{L\times d}$;
    \item an autoregressive Transformer decoder with cross-attention maps the generated embeddings back to discrete tokens.
\end{itemize}
% This design targets the trainability bottleneck noted for continuous DLMs (large decision space + decoding difficulty)~\citep{zhou2025ccdd} by explicitly learning a dedicated embedding-to-token decoder, while keeping diffusion entirely continuous. In particular, diffusion no longer needs to land exactly on token embeddings at every position; it only needs to generate continuous states that are \emph{decodable under context}, allowing the decoder to leverage linguistic regularities to resolve residual dependence and local ambiguity predicted by~\eqref{eq:tc_locality_decomp}.
This design addresses the previously noted trainability bottleneck by explicitly learning a dedicated embedding-to-token decoder, while keeping the diffusion process entirely continuous and allowing free choice of embeddings. In particular, diffusion no longer needs to land exactly on token embeddings at every position; it only needs to generate continuous states that are \emph{decodable under context}, allowing the decoder to leverage linguistic regularities to resolve residual dependence and local ambiguity predicted by~\eqref{eq:tc_locality_decomp}.

\section{\ourslong}% continuous diffusion language model in latent space

Let $\mathbf{y}=(y_1,\dots,y_L)$ be a token sequence with vocabulary size $|V|$.
A frozen embedding model $E: V^L \rightarrow \mathbb{R}^{L\times d}$ maps tokens to continuous
representations $\mathbf{x}_0 = E(\mathbf{y})$.
We denote the diffusion state at time $t\in[0,1]$ as $\mathbf{x}_t \in \mathbb{R}^{L\times d}$.
Our denoiser $f_\theta(\mathbf{x}_t,t)$ operates purely in the continuous space, while a separate
autoregressive decoder $p_\phi(\mathbf{y}\mid \hat{\mathbf{x}}_0)$ performs discrete realization.
Unless stated otherwise, $E(\cdot)$ is fixed and only $(\theta,\phi)$ are optimized. The framework of \ours is shown in \cref{fig:architecture}.

\subsection{Continuous Diffusion for Embedding Generation}

Given a tokenized sequence $\mathbf{y}=(y_1,\dots,y_L)$, we obtain continuous embeddings
$$
\mathbf{x}_0 = E(\mathbf{y}) \in \mathbb{R}^{L\times d},
$$

where $E(\cdot)$ can be an arbitrary pretrained text embedding model.

We define a variance-preserving (VP) continuous diffusion process\citep{song2020score} on embeddings by corrupting $\mathbf{x}_0$ with Gaussian noise:
$$
\mathbf{x}_t = \alpha(t)\mathbf{x}_0 + \sigma(t)\boldsymbol{\epsilon},\quad \boldsymbol{\epsilon}\sim\mathcal{N}(\mathbf{0},\mathbf{I}),
$$

where $t\sim \mathcal{U}(0,1)$, and $\alpha(t),\sigma(t)$ are a noise schedule satisfying $\alpha(t)^2+\sigma(t)^2=1$. We use cosine schedule proposed by \citet{iddpm}.

We use the velocity parameterization~\citep{salimans2022progressive} because it interpolates between predicting noise and predicting data, and has been shown to improve stability—especially when sampling with few steps—relative to
direct $\epsilon$-prediction. Concretely, we train the denoiser to predict the \textit{velocity}:
$$
\mathbf{v}_t \triangleq \alpha(t)\boldsymbol{\epsilon} - \sigma(t)\mathbf{x}_0,
$$
We parameterize the denoiser as
$$
\hat{\mathbf{v}}_\theta = f_\theta(\mathbf{x}_t, t),
$$
and recover estimates of $\mathbf{x}_0$ and $\boldsymbol{\epsilon}$ via 
$$
\hat{\mathbf{x}}_0 = \alpha(t)\mathbf{x}_t - \sigma(t)\hat{\mathbf{v}}_\theta,\qquad
\hat{\boldsymbol{\epsilon}} = \sigma(t)\mathbf{x}_t + \alpha(t)\hat{\mathbf{v}}_\theta.
$$

Under the VP formulation with v‑prediction, the denoiser’s objective is to align its output $\hat{\mathbf{v}}_\theta$ with the ground‑truth velocity $\mathbf{v}_t = \alpha(t)\boldsymbol{\epsilon} - \sigma(t)\mathbf{x}_0$. We optimize the following velocity prediction loss:

$$
\mathcal{L}_{\text{diff}}(\theta)
= \mathbb{E}_{t\sim \mathcal{U}(0,1),\,\mathbf{x}_0,\,\boldsymbol{\epsilon}\sim\mathcal{N}(0,I)}
\left[w(t)\cdot \left\lVert f_\theta(\mathbf{x}_t, t) - \mathbf{v}_t\right\rVert_2^2\right],
$$
where $w(t)$ is a time‑dependent weighting function that can improve optimization stability (for example, constant weighting or SNR‑based schemes used in image diffusion literature\citep{ddpm, iddpm, karras2022elucidating}).

\subsection{Contextualized Rounding with AR Decoder}
To map the generated continuous embeddings back to text, we employ a Transformer decoder $p_\phi$ that uses cross‑attention over the denoised embedding sequence. Formally, the conditional likelihood of a token sequence $\mathbf{y}=(y_1,\dots,y_L)$ given the recovered embeddings is factorized as
$$
p_\phi(\mathbf{y}\mid \hat{\mathbf{x}}_0)=\prod_{i=1}^L p_\phi(y_i \mid y_{< i}, \hat{\mathbf{x}_0}),
$$

This two-stage strategy first runs diffusion in a continuous latent space, and then decodes the result into discrete text. This is similar to the latent diffusion for language setup in LD4LG~\citep{ld4lg}, which combines a latent model with a encoder-decoder language model to produce text from latent vectors. Unlike LD4LG, we do not need to rely on an encoder–decoder language model and can use off-the-shelf, state-of-the-art text embedding models instead.

Under teacher forcing, the decoder is trained to reconstruct the ground‑truth tokens from the (ideally denoised) continuous embeddings with the standard cross‑entropy objective:
$$
\mathcal{L}(\phi;\theta)=
\mathbb{E}\left[-\sum_{i=1}^{L}\log p_\phi(y_i\mid y_{<i}, \mathbf{x}_0)\right].
$$
To improve robustness and generalization to the imperfect outputs produced by the diffusion model, we follow the noise‑augmentation strategy from Representation Autoencoder~\citep{rae} training: we add a small Gaussian perturbation $\mathbf{n}\sim\mathcal{N}(0,\sigma^2\mathbf{I})$ to the recovered embeddings during decoder training. This encourages the decoder to tolerate and correct slight deviations from clean embeddings

The resulting noise‑augmented decoding loss becomes
$$
\mathcal{L}_{\text{dec}}(\phi;\theta)=
\mathbb{E}\left[-\sum_{i=1}^{L}\log p_\phi(y_i\mid y_{<i}, \mathbf{x}_0 + \mathbf{n})\right],
$$
which trains the decoder to be resilient to small perturbations in its conditioning embeddings and better aligned with the outputs of the diffusion denoiser.

\subsection{Inference}
At inference time, generation proceeds in two stages. First, we sample a continuous embedding sequence by running the learned reverse diffusion process: starting from Gaussian noise $\mathbf{x}_{t=1}\sim\mathcal{N}(0,I)$, we iteratively apply the denoiser $f_\theta(\mathbf{x}_t,t)$ (with a chosen numerical solver) to progressively remove noise and obtain a final denoised embedding sequence $\mathbf{x}_0\in\mathbb{R}^{L\times d}$. Second, we translate these hidden states into tokens using the autoregressive Transformer decoder $p_\phi$ conditioning on $\hat{\mathbf{x}}_0$ via cross-attention. The decoder samples $y_i \sim p_\phi(y_i \mid y_{<i}, \hat{\mathbf{x}}_0)$ until an end-of-sequence token is produced (or a length limit is reached). In this way, diffusion handles global continuous-sequence generation, while the AR decoder performs discrete token realization from the generated hidden states.
% \begin{algorithm}[tb]
%   \caption{Unconditional Generation}
%   \label{alg:uncond_gen}
%   \begin{algorithmic}
%     \REQUIRE Time grid $1 = t_0 > t_1 > \dots t_K = 0$, Gen. Length $L$, batch size $b$, hidden size $h$
%     \STATE Initialize: $x_t \sim \mathcal{N}(\mathbf{0},\mathbf{I}) \in \mathbf{R}^{b\times L\times h}$, tokens $ = \text{[bos]} \in \mathbf{Z}^{b\times 1}$
%     \FOR{$k=0$ {\bfseries to} $K$}
%         \STATE diffusion
%     \ENDFOR
%     \FOR{$l=1$ {\bfseries to} $L$}
%         \STATE ar
%     \ENDFOR
%   \end{algorithmic}
% \end{algorithm}

\section{Experiment}
% \begin{table*}[ht!]
% \renewcommand\arraystretch{1.15}
% \centering
% \caption{Unconditional generation on OpenWebText. We compare \ours to discrete baselines (MDLM and SEDD) using 1,000 generated samples per method, reporting generative perplexity and an n-gram–based diversity score. We also include three diagnostic reference points: training set text, decoded ground-truth embeddings (“recovered training set”), and decoded Gaussian noise, to contextualize the metrics.\redmark{250 steps, tokenizer}}
% \label{tab:main_table1}
% \begin{tabular}{lc|ll}
% \toprule
% Model                  & Decoder Temperature & Generative PPL($\downarrow$) & Diversity($\uparrow$) \\ \midrule
% Training set           & -                   & 16.75   & 0.2191    \\
% Recovered training set & -                   & 25.07   & 0.2282    \\
% Decoded noise          & -                   & 14.24   & 0.0380    \\ \midrule
% MDLM$^*$~\citep{mdlm}& -                   & 123.73  & 0.4784    \\
% SEDD~\citep{sedd}& -                   & 129.57  & 0.4742    \\ \midrule
% \ours\redmark{single column}        & 0.00                & 47.71   & 0.1660    \\
%                        & 0.25                & 50.68   & 0.1937    \\
%                        & 0.50                & 66.31   & 0.2670    \\
%                        & 0.75                & 109.80  & 0.3718    \\
%                        & 1.00                & 164.90  & 0.4842    \\ \bottomrule
% \end{tabular}
% \end{table*}

\begin{table}[ht!]
\renewcommand\arraystretch{1.2}
\centering
\caption{Unconditional generation on OpenWebText. We compare \ours to discrete baselines (MDLM and SEDD) using 1,000 generated samples per method, reporting generative perplexity and an n-gram–based diversity score. We also include three diagnostic reference points: training set text, decoded ground-truth embeddings (“recovered training set”), and decoded Gaussian noise, to contextualize the metrics. We use 250 sampling steps for all runs while varying the decoder temperature for \ours.}
\label{tab:main_table1}
\begin{tabular}{l|ll}
\toprule
Model                  & Gen. PPL($\downarrow$) & Diversity($\uparrow$) \\ \midrule
Training set           & 16.75                        & 0.2191                \\
Recovered training set & 25.07                        & 0.2282                \\
Decoded noise          & 14.24                        & 0.0380                \\ \midrule
MDLM$^*$  & 123.73                       & 0.4784                \\
SEDD      & 129.57                       & 0.4742                \\ \midrule
\ours($T=0.00$)        & 47.71                        & 0.1660                \\
\ours($T=0.25$)        & 50.68                        & 0.1937                \\
\ours($T=0.50$)        & 66.31                        & 0.2670                \\
\ours($T=0.75$)        & 109.80                       & 0.3718                \\
\ours($T=1.00$)        & 164.90                       & 0.4842                \\ \bottomrule
\end{tabular}
\end{table}
\subsection{Setup}
We evaluate the language modeling capability of our model via unconditional text generation, focusing on the quality of generated samples. We compare \ours against a latent diffusion language model (LD4LG~\citep{ld4lg}) and two strong discrete diffusion baselines, MDLM~\citep{mdlm} and SEDD~\citep{sedd}. For all comparisons, we report generative perplexity (Gen.\ PPL) as a proxy for fluency (lower is better), and an n-gram–based diversity metric as a proxy for lexical variety (higher is better) following previous work~\citep{ld4lg}.

To compare against the latent diffusion baseline LD4LG, we train \ours and LD4LG on the One Billion Word Benchmark (LM1B)~\citep{lm1b} using the standard data split. We train models using sequence length $L = 128$ with sentence packing. We train both models for 250k steps. For LD4LG, we use the BART variant and follow the baseline’s standard autoencoder setup and training procedure as reported in the original work.

For the comparison with discrete baselines, we train all models on the OpenWebText~\citep{owt} dataset. Since OpenWebText is substantially larger and more diverse than LM1B, it provides a stronger stress test for unconditional generation quality. We train \ours for 250k steps using context length $L=512$ with sequence packing. For a fair comparison, we retrain MDLM and SEDD following prior practice ~\citep{mdlm, sedd, duo} using standard training iterations (1M steps) but matching our context length of 512. We use the Qwen2 tokenizer~\citep{qwen3emb} for \ours and for baselines when possible. LD4LG cannot use the Qwen2 tokenizer because its BART-based autoencoder is tied to the BART vocabulary. MDLM does not work well with Qwen2 tokenizer (resulting in very high generative perplexity), which is consistent with observations in \citet{zhou2025ccdd}. We therefore use the GPT-2 tokenizer for MDLM (marked as $^*$ in \cref{tab:main_table1} and \cref{tab:abl_steps}).

\paragraph{Evaluation metrics.} For each method, we generate 1,000 unconditional samples and compute generative perplexity (Gen.\ PPL) computed by GPT2-large~\citep{gpt2} to measure fluency. To quantify diversity, we use the metric used in LD4LG, defined as the product of distinct n-gram ratios for $n \in \{2, 3, 4\}$:
$$
\text{Div} = \prod_{n=2}^4 \frac{|\text{unique $n$-grams}(\{\mathbf{w}_i\})|}{|\text{total $n$-grams}(\{\mathbf{w}_i\})|},
$$
where $\{\mathbf{w}_i\}$ is a set of generated samples.

To contextualize the generative perplexity (Gen.\ PPL) and diversity scores, we report three reference points.
\begin{itemize}
    \item Training set: Text drawn directly from the \emph{training set}, serving as a data-distribution anchor under our evaluation protocol.% achieves Gen.\ PPL $16.75$ and diversity $0.2191$, which serves as a data-distribution anchor under our evaluation protocol.
    \item Recovered training set: obtained by decoding ground-truth text embeddings, reflecting reconstruction error induced by the decoder. %yield Gen.\ PPL $25.07$ and diversity $0.2282$, indicating a non-trivial reconstruction gap attributable to the decoder/representation interface rather than the diffusion prior.
    \item Decoded noise: Decoding pure Gaussian noise with our contextualized decoder, which exhibit low Gen.\ PPL but extremely low diversity, indicative of a degenerate mode of highly repetitive text.%attains an apparently low Gen.\ PPL of $14.24$, but with extremely low diversity ($0.0380$), revealing a degenerate failure mode where the decoder can emit locally fluent yet highly repetitive text when unconditioned.
\end{itemize}
These probes show that (i) the decoder is capable of producing fluent text, but (ii) a meaningful generative model must match \emph{both} fluency and diversity, not fluency alone.

\paragraph{Training Details} For the diffusion model, we use the same model architecture as MDLM\citep{mdlm}. For the contextualized autoregressive decoder, we use the same model architecture as GPT2-small~\citep{gpt2} but with additional cross attention. The decoder is trained for 1 epoch on OpenWebText. We choose Qwen3-Embedding~\citep{qwen3emb} as \ours's pretrained embedding as it supports custom dimensions for the final embedding. We set the dimension of embeddings to 64 unless otherwise specified. All embeddings are normalized following \citet{rombach2022high}.

\begin{table}[H]
\renewcommand\arraystretch{1.2}
\centering
\caption{Unconditional generation on LM1b.}
\label{tab:main_table2}
\begin{tabular}{l|ll}
\toprule
Model & Gen.PPL($\downarrow$) & Diversity($\uparrow$) \\ \midrule
LD4LG & 167.47  & 0.5797    \\
\ours  & 104.76  & 0.3264    \\ \bottomrule
\end{tabular}
\end{table}
\subsection{Main Results}

On OpenWebText unconditional generation (\cref{tab:main_table1}), \ours spans a smooth fluency--diversity frontier by varying the decoder temperature while keeping the model and procedure fixed. At low temperatures, \ours is markedly more fluent than discrete baselines: Gen.\ PPL drops to $47.71$ ($T{=}0.00$) and $50.68$ ($T{=}0.25$), yet diversity remains non-trivial ($0.1660$--$0.1937$) and far from the collapse of decoded noise (diversity $0.0380$). As temperature increases, diversity rises monotonically ($0.2670 \rightarrow 0.3718 \rightarrow 0.4842$ for $T{=}0.50,0.75,1.00$) with a corresponding increase in Gen.\ PPL ($66.31 \rightarrow 109.80 \rightarrow 164.90$), forming a clear Pareto trade-off. Crucially, at $T{=}1.00$ \ours reaches diversity $0.4842$, matching or slightly exceeding MDLM ($0.4784$) and SEDD ($0.4742$), showing that \ours can operate in the same diversity regime as strong discrete counterparts, while offering substantially better fluency when the operating point favors it. On LM1B unconditional generation, we reuse the decoder trained with OpenWebText and the decoder's temperature is set to 1. The results are shown in \cref{tab:main_table2}, \ours significantly outperforms LD4LG in terms of fluency while maintaining nontrivial diversity.

\subsection{Ablations}
\subsubsection{Sampler and Sampling Steps}

As the diffusion process of \ours is fully continuous, we can leverage higher-order numerical solvers to improve few-step sampling. We compare standard ancestral sampling to DPM-Solver~\citep{dpmsolver} while varying the number of sampling steps from 250 down to 25. We fix the decoder temperature to 1.0 for all runs. The results are shown in \cref{tab:abl_sampler}.

Across all step budgets, DPM-Solver consistently yields better fluency (lower Gen.\ PPL) than ancestral sampling while maintaining high diversity. For example, at 250 steps, DPM-Solver reduces Gen.\ PPL from 164.90 (ancestral) to 147.53, with comparable diversity. The advantage is even more pronounced in the low-step regime: at 100 steps, DPM-Solver achieves 154.83 Gen.\ PPL versus 185.91 for ancestral sampling, while keeping diversity near 0.495. Even at 25 steps, DPM-Solver preserves strong diversity and slightly improves Gen.\ PPL (212.32 vs. 214.86). Overall, these results show that advanced solvers make \ours substantially more effective for fast generation, improving sample quality without sacrificing diversity.
\begin{table}[H]
\renewcommand\arraystretch{1.2}
\centering
\caption{Solver ablation for \ours. We compare ancestral sampling to DPM-Solver for unconditional generation across sampling budgets with decoder temperature set to 1.0.}
\label{tab:abl_sampler}
\begin{tabular}{l|lll}
\toprule
Solver                      & Sampling Steps & Gen.PPL($\downarrow$) & Div.($\uparrow$)   \\ \midrule
\multirow{4}{*}{Ancestral}  & 25                   & 214.86  & 0.4251 \\
                            & 50                   & 206.54 & 0.4734 \\
                            & 100                  & 185.91  & 0.4757 \\
                            & 250                  & 164.89 & 0.4842 \\ \midrule
\multirow{4}{*}{DPM-Solver} & 25                   & 212.32  & 0.4929 \\
                            & 50                   & 178.82  & 0.4942 \\
                            & 100                  & 154.83  & 0.4947 \\
                            & 250                  & 147.53  & 0.488  \\ \bottomrule
\end{tabular}
\end{table}
% \usepackage{multirow}
\begin{table}[h]
\renewcommand\arraystretch{1.2}
\centering
\caption{Few-step unconditional generation on OpenWebText. We compare our model sampled with DPM-Solver (decoder temperature $T=1$) against discrete diffusion baselines (MDLM, SEDD) under matched step budgets.}
\label{tab:abl_steps}
\begin{tabular}{l|lll}
\toprule
Sampling Steps       & Model     & Gen.PPL($\downarrow$) & Div.($\uparrow$)   \\ \midrule
\multirow{3}{*}{25}  & MDLM$^*$      & 232.78  & 0.5287 \\
                     & SEDD      & 221.63  & 0.5171  \\
                     & \ours & 212.32  & 0.4929 \\ \midrule
\multirow{3}{*}{50}  & MDLM$^*$      & 165.71  & 0.5046 \\
                     & SEDD      & 164.24  & 0.492 \\
                     & \ours      & 178.82  & 0.4942 \\ \midrule
\multirow{3}{*}{100} & MDLM$^*$      & 137.62  & 0.4877 \\
                     & SEDD      & 143.19  & 0.481 \\
                     & \ours & 154.83  & 0.4947 \\ \midrule
\multirow{3}{*}{250} & MDLM$^*$      & 123.73  & 0.4784 \\
                     & SEDD      & 131.96  & 0.4742 \\
                     & \ours & 147.53  & 0.488  \\ \bottomrule
\end{tabular}
\end{table}
Now we rival discrete baselines (MDLM/SEDD) in few-step generation by combining \ours with DPM-Solver. \cref{tab:abl_steps} reports Gen.\ PPL (↓) and diversity (↑) at matched sampling budgets (25/50/100/250 steps), with decoder temperature fixed to T=1 for \ours.

At the most aggressive budget of 25 steps, \ours achieves the best fluency among all methods, while retaining strong diversity. As the step budget increases, MDLM/SEDD become more fluent, but \ours remains competitive and stays in the same diversity regime: at 50–250 steps, \ours yields diversity around 0.49, comparable to the baselines. For instance, at 250 steps, \ours attains 0.488 diversity versus 0.478/0.474 for MDLM/SEDD, though with higher Gen.\ PPL. At intermediate budgets (100 steps), \ours reaches 154.83 Gen.\ PPL with 0.4947 diversity, surpassing the baselines’ diversity (0.4877/0.4892) while trailing in fluency.

Overall, these results highlight two takeaways: (i) thanks to the continuous formulation, \ours can exploit advanced solvers to enable high-quality fast sampling, and (ii) in the few-step regime (especially 25 steps), \ours is already on par with or better than discrete diffusion baselines in fluency, while maintaining comparable diversity.
% \begin{table}[H]
% \renewcommand\arraystretch{1.2}
% \centering
% \caption{Ablation on dimension of hidden states\redmark{solver}}
% \label{tab:abl_dim}
% \begin{tabular}{l|ll}
% \toprule
% Hidden Size & Gen.PPL & Div.   \\ \midrule
% 768         & 523.07  & 0.5764 \\
% 256         & 294.42  & 0.5056 \\
% 64          & 164.90  & 0.4842 \\ \bottomrule
% \end{tabular}
% \end{table}

\begin{table}[H]
\renewcommand\arraystretch{1.2}
\centering
\caption{Ablation on dimension of hidden states.}
\label{tab:abl_dim}
\begin{tabular}{c|l|ll}
\toprule
Hidden Size          & Solver     & Gen.PPL($\downarrow$) & Div.($\uparrow$)   \\ \midrule
\multirow{2}{*}{768} & Ancestral  & 523.07  & 0.5764 \\
                     & DPM-Solver & 546.10  & 0.6212 \\ \midrule
\multirow{2}{*}{256} & Ancestral  & 294.42  & 0.5056 \\
                     & DPM-Solver & 300.01  & 0.5475 \\ \midrule
\multirow{2}{*}{64}  & Ancestral  & 164.90  & 0.4842 \\
                     & DPM-Solver & 147.53  & 0.488  \\ \bottomrule
\end{tabular}
\end{table}
\subsubsection{Hidden State Dimension}

We study the effect of the hidden state dimension on the overall generation quality. Specifically, we vary $d \in \{64, 256, 768\}$ while keeping the decoder temperature to 1. As shown in \cref{tab:abl_dim}, increasing the hidden dimension does not translate into better text quality. While larger hidden states generally provide the decoder with higher representational capacity, they hinder the diffusion process. This leads to a degradation in overall text quality: the generative perplexity increases substantially from 164.90 at $d=64$ to 294.42 at $d=256$, and further to 523.07 at $d=768$. Interestingly, even the DPM-Solver, which is designed to handle diffusion more efficiently, struggles with the increased complexity brought by higher hidden dimensions. As the hidden state dimension grows, the DPM-Solver's generative perplexity surpasses that of the ancestral sampler, indicating that the solver cannot mitigate the added difficulty of a larger state space.

\subsubsection{Choice of Decoder}
\begin{table}[]
\renewcommand\arraystretch{1.2}
\centering
\caption{Ablation on choice of decoder}
\label{tab:abl_dec}
\begin{tabular}{l|ll}
\toprule
Decoder & Gen.PPL($\downarrow$) & Div.($\uparrow$)   \\ \midrule
Linear         & 153.44  & 0.1238 \\
Transformer Decoder  & 164.90  & 0.4842 \\ \bottomrule
\end{tabular}
\end{table}
% We analyze the impact of the decoder architecture on generation quality by comparing a linear decoder head with a Transformer-based decoder, as summarized in \cref{tab:abl_dec}. As discussed in \cref{sec:empirical}, linear heads exhibit weak performance in token recovery. Here we further examine whether this limitation also affects full text generation. From the perspective of perplexity, the linear decoder achieves a slightly lower Gen.PPL (153.44) than the Transformer decoder (164.90). However, this apparent advantage is misleading when considered alongside diversity. The linear decoder yields an extremely low diversity score (0.1238), indicating severe repetition and mode collapse in the generated text. In practice, this manifests as repetitive phrases and limited lexical or structural variation, substantially degrading perceived text quality despite the favorable perplexity. In contrast, the Transformer decoder produces significantly higher diversity (0.4842), while maintaining competitive Gen.PPL. This suggests that the Transformer’s autoregressive modeling capacity and contextualized representations are crucial for text quality in generation. Overall, these results confirm that linear decoders are inadequate for generation, as their lack of expressive power leads to pathological repetition.
We study the effect of decoder architecture by comparing a linear head and a Transformer decoder (\cref{tab:abl_dec}). As noted in \cref{sec:empirical}, linear heads perform poorly in token recovery; we examine whether this limitation extends to text generation. Although the linear decoder attains a slightly lower Gen.PPL (153.44 vs. 164.90), it suffers from extremely low diversity (0.1238), indicating severe repetition and mode collapse, which substantially degrades text quality in practice. In contrast, the Transformer decoder achieves much higher diversity (0.4842) while maintaining competitive perplexity, highlighting the importance of contextual modeling for high quality decoding. Overall, linear decoders are inadequate for generation due to their limited expressive capacity.

\section{Conclusion}
In this work, we argue that the performance gap between continuous and discrete diffusion language models arises primarily from a decoding rounding rather than from limitations of continuous diffusion itself. Through theoretical analysis and controlled token-recovery experiments, we show that rounding is inherently sequence dependent and that pointwise linear heads are provably suboptimal for mapping continuous representations back to tokens. Building on this insight, we propose \ours, a two-stage framework that performs diffusion entirely in an embedding space while using a context aware autoregressive Transformer decoder to realize discrete tokens. Experiments on LM1B and OpenWebText demonstrate that \ours substantially improves over latent diffusion baselines and becomes competitive with strong discrete DLMs, while exposing a simple decoder temperature mechanism to smoothly trade off fluency and diversity. Together, these results suggest that continuous diffusion and discrete language modeling are complementary rather than competing, and that treating rounding as a contextual problem unlocks much of the unrealized potential of continuous diffusion language models.

\section*{Acknowledgement}
This work is sponsored by the National Natural Science Foundation of China (NSFC) grant (No. 62576211) the National Key Research and Development Program of China (No. 2023ZD0121402), and the Specialized Program on Fundamental Research from Science and Technology Commission of Shanghai Municipality (No. 2025SHZDZX025G09).

\bibliography{references}
\appendix
\section{More Experimental Details}
The decoder of \ours is trained for 1 epoch with a batch size of 512. The Adam optimizer was used with a learning rate of $1.0 \times 10^{-3}$, weight decay of $1.0 \times 10^{-1}$, $\beta_1=0.9$, $\beta_2=0.95$, and gradient clipping at a maximum norm of 1.0. The learning rate followed a cosine annealing schedule, with a warmup ratio of 5\%.
For the DiT training, we also use a batch size of 512 and Adam optimizer. The learning rate is set to $4.0\times 1.0^{-4}$ with a constant learning rate schedule. The weight decay is set to $0.02$, $\beta_1=0.9$, $\beta_2=0.95$, and gradient clipping at a maximum norm of 1.0. The warmup steps is set to 10000.
\section{Proof}
\label{sec:proof}
\paragraph{Detailed proof of Proposition ~\ref{prop}.}
Let $p(x,y)$ denote the true joint distribution of $(X,Y)$, and let $\Delta$ be the probability simplex over $\mathcal{V}^L$ (all length-$L$ token sequences). We assume each candidate decoder $q(\cdot\mid x)$ is a valid conditional distribution in $\Delta$ for $p$-a.e.\ $x$.

\begin{proof}
We prove the two minimizers and then the lower bound by conditional total correlation.

\smallskip
\noindent\textbf{Step 1: Bayes-optimal decoder over $\mathcal{D}_{\mathrm{seq}}$.}
Fix any conditional distribution $q(y\mid x)$. The expected NLL risk can be rewritten by conditioning on $X$:
\begin{equation}
\mathcal{R}(q)
= \mathbb{E}_{X}\Big[\mathbb{E}_{Y\mid X}\big[-\log q(Y\mid X)\big]\Big]
= \mathbb{E}_{X}\Big[\sum_{y} p(y\mid X)\big(-\log q(y\mid X)\big)\Big].
\end{equation}
For each fixed value $x$, define the \emph{cross-entropy} between $p(\cdot\mid x)$ and $q(\cdot\mid x)$:
\begin{equation}
H\!\big(p(\cdot\mid x),q(\cdot\mid x)\big)
:= -\sum_{y} p(y\mid x)\log q(y\mid x).
\end{equation}
A standard identity (cross-entropy decomposition) states that for any two distributions $P,Q$ on the same space,
\begin{equation}
H(P,Q)=H(P)+D_{\mathrm{KL}}(P\|Q),
\end{equation}
hence, applying it pointwise at each $x$ yields
\begin{equation}
-\sum_{y} p(y\mid x)\log q(y\mid x)
=
H\!\big(p(\cdot\mid x)\big)
+
D_{\mathrm{KL}}\!\big(p(\cdot\mid x)\,\|\,q(\cdot\mid x)\big).
\end{equation}
Taking expectation over $X$ gives
\begin{equation}
\mathcal{R}(q)
=
\underbrace{\mathbb{E}_{X}\big[H(p(\cdot\mid X))\big]}_{H(Y\mid X)}
+
\mathbb{E}_{X}\Big[D_{\mathrm{KL}}\!\big(p(\cdot\mid X)\,\|\,q(\cdot\mid X)\big)\Big].
\label{eq:seq_risk_decomp}
\end{equation}
Since KL divergence is nonnegative and equals $0$ iff the two distributions agree almost surely, the second term in~\eqref{eq:seq_risk_decomp} is minimized (to $0$) by choosing
\begin{equation}
q^*(y\mid x)=p(y\mid x)\quad\text{for $p$-a.e.\ $x$}.
\end{equation}
Therefore,
\begin{equation}
\min_{q\in\mathcal{D}_{\mathrm{seq}}}\mathcal{R}(q)=H(Y\mid X).
\end{equation}
(These facts follow from the cross-entropy/KL identity and nonnegativity of KL; see, e.g., standard information theory references. \citep{ } ) % cite in bib; see note below

\smallskip
\noindent\textbf{Step 2: Bayes-optimal decoder over $\mathcal{D}_{\mathrm{pw}}$.}
Now restrict to pointwise-factorized decoders
\begin{equation}
q(y\mid x)=\prod_{i=1}^L q_i(y_i\mid x_i),
\qquad x=(x_1,\dots,x_L).
\end{equation}
Then the log-likelihood separates:
\begin{equation}
-\log q(Y\mid X) = -\sum_{i=1}^L \log q_i(Y_i\mid X_i),
\end{equation}
and by linearity of expectation,
\begin{equation}
\mathcal{R}(q)
=\sum_{i=1}^L \mathbb{E}\big[-\log q_i(Y_i\mid X_i)\big].
\label{eq:pw_risk_sum}
\end{equation}
For each position $i$, the inner expectation depends only on the joint law of $(X_i,Y_i)$. Repeating the same cross-entropy decomposition as in Step 1 but for the conditional distribution $p(y_i\mid x_i)$ gives, for any $q_i(\cdot\mid x_i)$,
\begin{equation}
\mathbb{E}\big[-\log q_i(Y_i\mid X_i)\big]
=
H(Y_i\mid X_i)
+
\mathbb{E}_{X_i}\Big[D_{\mathrm{KL}}\!\big(p(\cdot\mid X_i)\,\|\,q_i(\cdot\mid X_i)\big)\Big],
\label{eq:pos_decomp}
\end{equation}
where $p(\cdot\mid X_i)$ abbreviates $p(Y_i=\cdot\mid X_i)$. The KL term is minimized to $0$ iff
\begin{equation}
q_i^*(y_i\mid x_i)=p(y_i\mid x_i)\quad\text{for $p$-a.e.\ $x_i$}.
\end{equation}
Plugging into~\eqref{eq:pw_risk_sum} yields
\begin{equation}
\min_{q\in\mathcal{D}_{\mathrm{pw}}}\mathcal{R}(q)
=
\sum_{i=1}^L H(Y_i\mid X_i).
\end{equation}

\smallskip
\noindent\textbf{Step 3: The exact optimality gap.}
Subtracting the two minima obtained above gives the equality:
\begin{equation}
\min_{q\in \mathcal{D}_{\mathrm{pw}}}\mathcal{R}(q)\;-\;\min_{q\in \mathcal{D}_{\mathrm{seq}}}\mathcal{R}(q)
=
\sum_{i=1}^L H(Y_i\mid X_i)\;-\;H(Y\mid X).
\end{equation}

\smallskip
\noindent\textbf{Step 4: Lower bound by conditional total correlation and nonnegativity.}
By the definition of conditional total correlation,
\begin{equation}
\mathrm{TC}(Y\mid X)
=
\mathbb{E}_{X}\!\left[
D_{\mathrm{KL}}\!\left(p(Y\mid X)\,\Big\|\,\prod_{i=1}^L p(Y_i\mid X)\right)
\right]
=
\sum_{i=1}^L H(Y_i\mid X)-H(Y\mid X)
\;\ge\;0,
\end{equation}
where nonnegativity follows from KL $\ge 0$. Moreover, since $X$ contains at least as much information as $X_i$,
conditioning reduces entropy:
\begin{equation}
H(Y_i\mid X)\;\le\;H(Y_i\mid X_i)\qquad\forall i.
\end{equation}
Therefore,
\begin{equation}
\sum_{i=1}^L H(Y_i\mid X_i)-H(Y\mid X)
\;\ge\;
\sum_{i=1}^L H(Y_i\mid X)-H(Y\mid X)
=
\mathrm{TC}(Y\mid X)
\;\ge\;0,
\end{equation}
which proves Proposition~\ref{prop}.
\end{proof}

% \paragraph{References for standard identities (for your bibliography).}
% The proof uses only two standard facts:
% (i) cross-entropy decomposition $H(P,Q)=H(P)+D_{\mathrm{KL}}(P\|Q)$ and KL $\ge 0$ (e.g., Chapter 2 of Cover--Thomas \citep{cover_thomas_elements}); and
% (ii) conditioning reduces entropy $H(U\mid V,W)\le H(U\mid V)$ (standard information theory lecture notes \citep{langer_lecture16}).
% The definition/entropy form of total correlation is standard (e.g., \citep{watanabe_total_correlation}).

% ====================
% BIBLIOGRAPHY
% ====================

\end{document}